\documentclass[letterpaper, 10 pt, conference]{ieeeconf}  %

\IEEEoverridecommandlockouts                              %

\usepackage{amsmath}
\usepackage{amsfonts}
\usepackage{mathtools}
\usepackage{amssymb}

\usepackage{amsthm}
\usepackage{algorithm}
\usepackage{algpseudocode}
\usepackage{xcolor}
\usepackage[font={small}]{caption}
\usepackage{graphicx}
\usepackage{balance}
\usepackage{booktabs}
\usepackage{wrapfig}
\usepackage{subfigure}
\usepackage{url}
\usepackage[hidelinks]{hyperref}

\newtheorem{theorem}{Theorem}
\newtheorem{definition}{Definition}

\newtheorem{remark}{Remark}

\renewcommand{\baselinestretch}{0.96} 

\title{\LARGE \bf
Receding Horizon Planning with Rule Hierarchies \\ for Autonomous Vehicles
}

\author{Sushant Veer, Karen Leung, Ryan K. Cosner, Yuxiao Chen, Peter Karkus, and Marco Pavone%
\thanks{Sushant Veer, Yuxiao Chen, and Peter Karkus are with NVIDIA Research {\tt\small\{\{sveer,yuxiaoc,pkarkus\}@nvidia.com\}}. Karen Leung is with the University of Washington and NVIDIA Research {\tt\small \{kymleung@uw.edu, kaleung@nvidia.com\}}. Ryan Cosner is with the California Institute of Technology {\tt\small \{rkcosner@caltech.edu\}} (this work was conducted while Ryan was an intern at NVIDIA Research). Marco Pavone is with Stanford University and NVIDIA Research {\tt\small \{pavone@stanford.edu, mpavone@nvidia.com\}}.
}}

\begin{document}

\maketitle
\thispagestyle{empty}
\pagestyle{empty}

\begin{abstract}
Autonomous vehicles must often contend with conflicting planning requirements, e.g., safety and comfort could be at odds with each other if avoiding a collision calls for slamming the brakes. To resolve such conflicts, assigning importance ranking to rules (i.e., imposing a rule hierarchy) has been proposed, which, in turn, induces rankings on trajectories based on the importance of the rules they satisfy. On one hand, imposing rule hierarchies can enhance interpretability, but introduce combinatorial complexity to planning; while on the other hand, differentiable reward structures can be leveraged by modern gradient-based optimization tools, but are less interpretable and unintuitive to tune. In this paper, we present an approach to equivalently express rule hierarchies as differentiable reward structures amenable to modern gradient-based optimizers, thereby, achieving the best of both worlds. We achieve this by formulating \emph{rank-preserving reward functions} that are monotonic in the rank of the trajectories induced by the rule hierarchy; i.e., higher ranked trajectories receive higher reward. Equipped with a rule hierarchy and its corresponding rank-preserving reward function, we develop a two-stage planner that can efficiently resolve conflicting planning requirements. We demonstrate that our approach can generate motion plans in $\sim$7-10 Hz for various challenging road navigation and intersection negotiation scenarios. Our code for building STL rule hierarchies is made available at: \url{https://github.com/NVlabs/rule-hierarchies}.
\end{abstract}

\section{Introduction}
Autonomous Vehicles (AVs) must satisfy a plethora of rules pertaining to safety, traffic rules, passenger comfort, and progression towards the goal. These rules often, unfortunately, conflict with each other when unexpected events occur. For instance, avoiding collision with a stationary or dangerously slow non-ego vehicle on the highway might necessitate swerving on to the shoulder, violating the traffic rule to keep the shoulder clear. To resolve the juxtaposing requirements posed by these rules, ordering them according to their importance in a hierarchy was proposed in \cite{censi2019liability}, which induces a ranking on trajectories. Trajectories that satisfy higher importance rules in the hierarchy are ranked higher than those trajectories which satisfy lower importance rules; see Fig.~\ref{fig:anchor} for an illustration of rule hierarchies.

Rule hierarchies provide a systematic approach to plan motions that prioritize more important rules (e.g., safety) over the less important ones (e.g., comfort) in the event that all rules cannot be simultaneously satisfied. Furthermore, they offer greater transparency in planner design and are more amenable to introspection. However, rule hierarchies introduce significant combinatorial complexity to planning; in the worst-case, an $N$-level rule hierarchy could require solving $2^N$ optimization problems. An alternate to rule hierarchies is to plan with a ``flat" differentiable reward function comprised of weighted contributions from all the rules. The weights can be tuned manually or using data \cite{levine2012continuous,ng2000algorithms}. Although such flat differentiable reward structures are amenable to planning with standard optimization tools, they are less interpretable and diagnosable. AV developers are, therefore, faced with the choice between more interpretable but challenging to plan with rule heirarchies and more opaque but easy to plan with standard reward structures. In this paper, we show that rule hierarchies can, in fact, be unrolled in a principled manner into flat differentiable rewards by leveraging recent developments in differentiable Signal Temporal Logic (STL) representations \cite{leung2020back} and the notion of rank-preserving reward functions that we introduce. Hence, the flat differentiable representation of hierarchies achieves the best of both worlds: the expressive power of hierarchies and the computational advantages of gradient-based optimization through modern optimization tools.

\begin{figure}[t]
  \centering
    \includegraphics[width=0.45\textwidth]{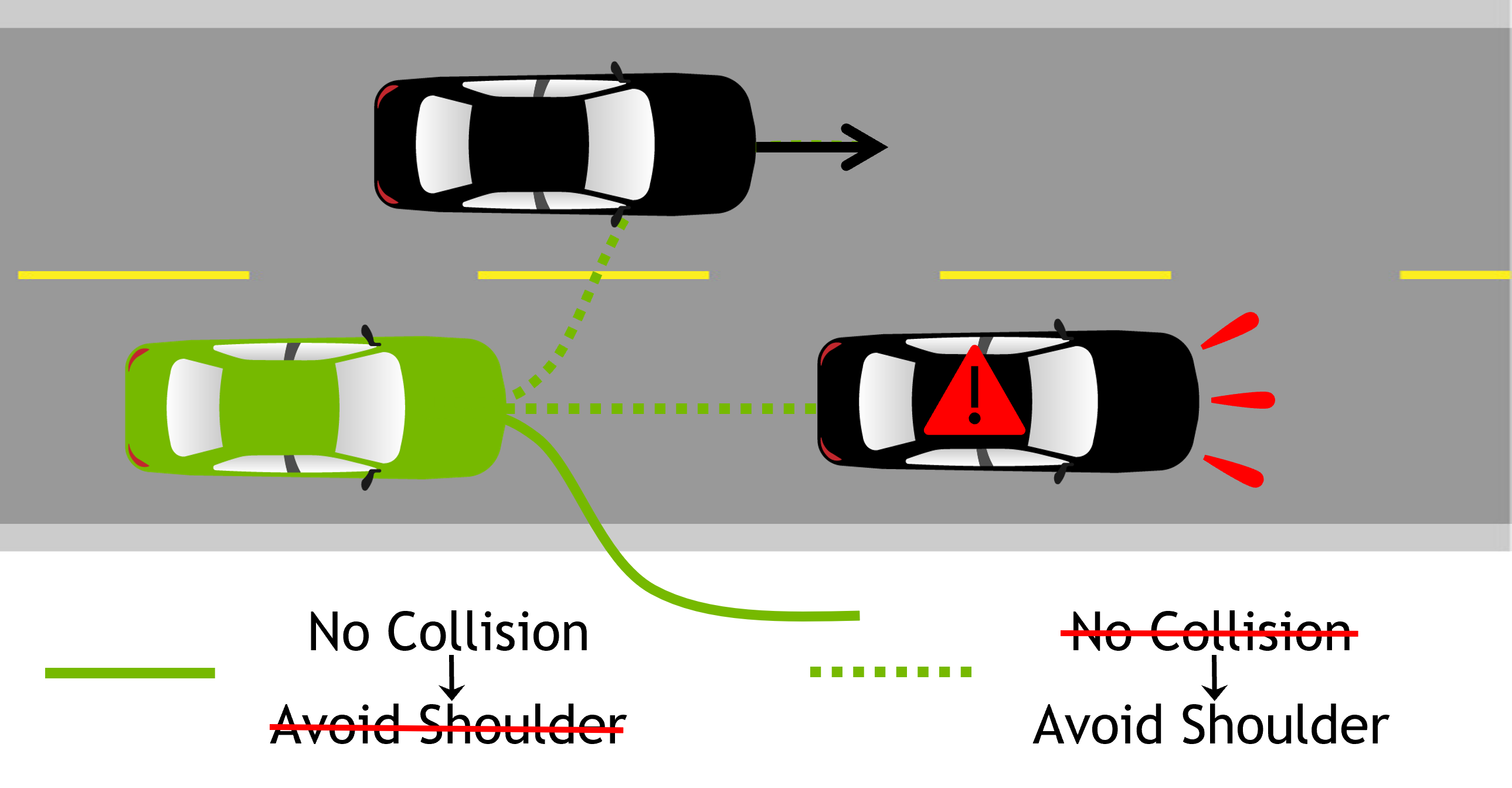}
    \vspace{-2mm}
  \caption{Illustration of rule hierarchy. Ego vehicle is green and non-ego vehicles are black. The rule hierarchy has two rules ordered in decreasing importance: (i) No collision, (ii) Avoid shoulder. Non-ego vehicle with exclamation suddenly brakes. Ego trajectory that prevents collision but does not avoid the shoulder (solid green) has a higher rank than the trajectories that collide (dashed green). \label{fig:anchor}}
  \vspace{-5mm}
\end{figure}

Our approach for planning with rule hierarchies only solves \emph{two} optimization problems (instead of the worst-case $2^N$) regardless of the choice of $N$. The key to achieving this is the formulation of a \emph{rank-preserving reward function} $R$ that exhibits the following property: higher ranked trajectories that satisfy more important rules receive a higher reward compared to trajectories that only satisfy lower importance rules. Maximizing this reward directly allows our planner to choose trajectories that have a higher rule-priority ordering without checking all combinations of rules. However, this reward function is highly nonlinear and suffers from an abundance of local maxima. To overcome this challenge, we use a two-stage planning approach: the first stage searches through a finite set of primitive trajectories and selects the trajectory that satisfies the highest priority rules; the second stage warm starts a continuous optimization for maximizing the reward $R$ with the trajectory supplied by the first-stage. To the best of our knowledge, this is the first work that plans for AVs in real-time with rule hierarchies. Prior work \cite{xiao2021rule,xiao2021ruleJournal} considered tracking an a priori given reference trajectory under rule hierarchies; in this paper, we do not assume the availability of an a priori reference trajectory.

\noindent \textbf{Statement of Contributions.} Our contributions are threefold: (i) We introduce the notion of a rank-preserving reward function in Definition~\ref{def:rank} and provide a systematic approach to constructing such a function for any given rule hierarchy in Theorem~\ref{thm:rank-preserve} and Remark~\ref{rem:diff-reward}. (ii) We present a two-stage receding-horizon planning approach that leverages the rank-preserving reward to efficiently plan with rule hierarchies and operates at a frequency of around $\sim$7-10 Hz. (iii) We demonstrate the ability of the rank-preserving reward based planner to rapidly adapt to challenging road navigation and intersection negotiation scenarios online without any scenario-specific hyperparameter tuning.

\section{Related Works}

\noindent \textbf{Tuning Reward/Cost Functions.} Inverse optimal control \cite{levine2012continuous} and inverse reinforcement learning \cite{ng2000algorithms} offer approaches to tune the contribution of competing criteria in the total reward/cost function by harnessing data. Various recent works use data to learn such reward functions for autonomous vehicles by learning a trajectory scoring function \cite{phan2022driving}, for planners that combine behavior generation and local motion planning \cite{rosbach2019driving}, and by leveraging trajectory preferences of an expert \cite{Sadigh2017active,katz2021preference}, among many others. Due to the challenges associated with learning reward functions and then using them in traditional trajectory optimizers, some approaches directly learn a data-driven policy via imitation learning \cite{vitelli2022safetynet} or reinforcement learning \cite{cao2020reinforcement}. In this paper, we present a method that, given a rule hierarchy, analytically provides a ``weighting" of the reward terms associated with each rule without requiring any additional data; see Theorem~\ref{thm:rank-preserve}.

\noindent \textbf{Hierarchical Specifications.} The use of hierarchies for negotiating between various criteria in an optimization was proposed as early as 1967 \cite{waltz1967engineering} to avoid tuning the weights on individual cost functions in their aggregated weighted sum. Hierarchical optimizations have also appeared in the field of logic programming \cite{wilson1993hierarchical} to ``optimally" negotiate between strict requirements and non-strict preferences. Recently, satisfiability modulo theories \cite{barrett2018satisfiability,shoukry2017linear} have been used to detect \cite{zhang2021systematic} and react to traffic-rule violations \cite{lin2022rule}. Notions of maximum satisfaction and minimum violation of rules \cite{tuumova2013minimum,dimitrova2018maximum,mehdipour2020specifying} have also been studied in the Linear Temporal Logic (LTL) and Signal Temporal Logic (STL) literature. Philosophically our paper shares the approach of formulating a scalar-valued function that reflects the rank of rule satisfaction with \cite{tuumova2013minimum} and \cite{mehdipour2020specifying}. However, \cite{tuumova2013minimum} considers finite-state dynamics, while the conjunction aggregation in \cite[Equation (7)]{mehdipour2020specifying} does not satisfy the \emph{strict} rank-preserving property we enforce in Definition~\ref{def:reward}; i.e., if the rank of a trajectory is higher, then its reward should be strictly higher. The strictness of Definition~\ref{def:reward} facilitates efficient planning by clearly distinguishing trajectories with different ranks. Unlike the papers discussed above, we formulate a reward function in Theorem~\ref{thm:rank-preserve} which satisfies the strict Definition~\ref{def:reward} and then use it to plan motions for AVs in complex driving scenarios.

\noindent \textbf{Rule Hierarchies for AVs.} Recently, rulebooks---which are a pre-ordering of rules---for AVs were proposed in \cite{censi2019liability} as ``blue-prints" for desirable driving preferences. The rulebook, being a pre-ordering, can be adapted to specific geographic and cultural driving norms into total-order rulebooks for planning and control while retaining the hierarchy structure prescribed by the pre-order. The consistency and completeness of rulebooks were investigated in \cite{phan2019towards} using formal methods, \cite{collin2020safety} used rulebooks for verification and validation of motion plans, and \cite{helou2021reasonable} used rulebooks to model human-driving behavior. A control strategy to track an a priori given reference trajectory while satisfying a total-order rulebook was presented in \cite{xiao2021rule,xiao2021ruleJournal}. We remark that none of the above mentioned papers generate motion plans for AVs that adhere to rule hierarchies \emph{online}, unlike our paper.

\section{Problem Formulation}

\noindent \textbf{Ego Dynamics.} We refer to the AV as the ego vehicle. Let $x\in\mathcal{X}\subseteq\mathbb{R}^n$ be the ego's state and $u\in\mathcal{U}\subseteq\mathbb{R}^m$ be the control inputs. The ego exhibits discrete-time dynamics:
\begin{align}\label{eq:dyn}
    x_{t+1} = f(x_t,u_t) ,
\end{align}
where $t$ represents the discrete-time and $f:\mathcal{X}\times\mathcal{U}\to\mathcal{X}$ is continuously differentiable. Let $\mathcal{S}$ represent the space of trajectories generated by \eqref{eq:dyn} starting from any admissible initial state $x_0\in\mathcal{X}$ and evolving under the influence of any control sequence $u_{0:T}$ spanning a horizon of $T$ time-steps.

\noindent \textbf{Non-ego Agents and Map.} We denote the state of non-ego agents by $x_{\rm ne}$ and the world map which contains information regarding lane lines, stop signs, etc. by $x_{\rm map}$. For notational convenience, we augment non-ego agent state trajectories $x_{\mathrm{ne},0:T}$ and the world map $x_{\rm map}$ to create the world scene $w:=(x_{\mathrm{ne},0:T}, x_{\rm map})\in\mathcal{W}$. In practice, the planner will receive a predicted world scene $\hat{w}$ which is generated from maps and by the prediction \cite{schmerling2018multimodal,salzmann2020trajectron++} modules in an AV stack. However, dealing with these modules is beyond the scope of this paper which focuses on planning with rule hierarchies. Therefore, in the interest of staying focused on the planner, we assume availability of the world scene $w$. We discuss possible extensions with predicted $\hat{w}$ in Section~\ref{sec:conclusions}.

\noindent \textbf{Rules and Rule Robustness.} We define a rule $\phi : \mathcal{S}\times\mathcal{W}\to\{\mathtt{True,False}\}$ as a boolean function that maps an ego trajectory $x_{0:T}\in\mathcal{S}$ and the world scene $w\in\mathcal{W}$ to \texttt{True} or \texttt{False}, depending on whether the trajectory satisfies the rule or not. The robustness $\hat{\rho}_i: \mathcal{S}\times\mathcal{W} \to \mathbb{R}$ of a rule $\phi_i$ is a metric that provides the degree of satisfaction for a rule. The robustness is a positive scalar if the rule is satisfied and negative otherwise; more positive values indicate greater satisfaction of the rule while more negative values indicate greater violation. We express the rules as STL formulae \cite{maler2004monitoring} using STLCG \cite{leung2020back}, which comes equipped with backpropagatable robustness metrics. Expressing rules as STL formulae allows us to easily encode complex spatio-temporal specifications, such as stop in front of a stop sign for at least 1 second before driving on.

\noindent \textbf{Rule Hierarchy.} A rule hierarchy $\varphi$ is defined as a sequence of rules $\varphi:=\{\phi_i\}_{i=1}^N$ where the highest priority rule is indexed by $1$ and lowest by $N$. The robustness of a rule hierarchy $\varphi$ for a trajectory $x_{0:T}\in\mathcal{S}$ in a world scene $w\in\mathcal{W}$ is an $N$-dimensional vector-valued function whose elements comprise the robustness of the $N$-rules in $\varphi$, expressed formally as $\hat{\rho}: (x_{0:T}, w) \mapsto (\hat{\rho}_1(x_{0:T}, w),\cdots,\hat{\rho}_N(x_{0:T}, w))$.

\noindent \textbf{Rank of a Trajectory.} The rule hierarchy gives rise to a total 
\begin{wraptable}{l}{0.20\textwidth}
    \centering
    \resizebox{0.20\textwidth}{!}{
        \begin{tabular}{cc}
            \toprule
            \textbf{Rank} & \textbf{Satisfied Rules}  \\ 
            \midrule
            1 & $\phi_1$, $\phi_2$, $\phi_3$ \\ 
            2 & $\phi_1$, $\phi_2$ \\ 
            3 & $\phi_1$, $\phi_3$ \\ 
            4 & $\phi_1$ \\ 
            5 & $\phi_2$, $\phi_3$ \\ 
            6 & $\phi_2$ \\ 
            7 & $\phi_3$ \\ 
            8 & $\emptyset$ \\ 
            \bottomrule
        \end{tabular}
    }
    \caption{Illustration of trajectory ranks for three rules.}
    \label{tab:order-example}
\end{wraptable}
order on trajectories. If a trajectory satisfies all the rules, then it has the highest rank in the order, while if it satisfies all rules but the lowest priority rule, then it has second rank, and so on; see Table~\ref{tab:order-example} for further clarity on trajectory ranking via a 3-rule example. Given a trajectory $x_{0:T}$ and the world scene $w$, let the robustness vector for the trajectory be defined as $\rho:=\hat{\rho}(x_{0:T},w)$. Using the robustness vector we formally define the rank of a trajectory:

\begin{definition}[Rank of a Trajectory]\label{def:rank}
Let $\varphi$ be a rule hierarchy with $N$ rules. Given a trajectory $x_{0:T}$ and the world scene $w$, let $\rho:=(\rho_1, \rho_2,\cdots\rho_N)$ be the robustness vector of the trajectory as defined above. Let $\mathrm{step}:\mathbb{R}\to\{0,1\}$ map negative real numbers to $0$ and all other real numbers to $1$. Then the rank $r:\mathbb{R}^N\to\{1, 2,\cdots, 2^N\}$ is defined as:
\begin{align}
    r(\rho):=2^N - \sum_{i=1}^N 2^{N-i} \mathrm{step}(\rho_i) .
\end{align}
\end{definition}

\noindent \textbf{Problem.} We want to solve the following optimization to obtain control inputs that result in a trajectory with the highest achievable rank in accordance with a rule hierarchy:
\begin{align} \label{eq:opt}
    \begin{aligned}
        \min_{u_{0:T}} & ~~~~ r\circ\hat{\rho}(x_{0:T},w) \\
        \text{s.t.} & ~~~~ x_{t+1} = f(x_t, u_t), ~\mathrm{for}~t=1,\cdots,T .
    \end{aligned}
\end{align}
In particular, we want to solve this problem efficiently without checking all $2^N$ combinations of rule satisfaction. Note that the constraints for this optimization, such as control bounds, can be baked in the rule hierarchy $\varphi$.

\section{Rank-Preserving Reward Function}

In this section, we present a differentiable rank-preserving reward function which enables us to circumvent the combinatorial challenges in solving \eqref{eq:opt} na\"ively. We first define the notion of a rank-preserving reward function that assigns higher rewards for trajectories with higher rank and lower rewards for trajectories with lower rank.

\begin{definition}[Rank-Preserving Reward Function] \label{def:reward}
A rank-preserving reward function $R:\rho\mapsto R(\rho)\in \mathbb{R}$ satisfies:
\begin{align}
   r(\rho) < r(\rho') \implies R(\rho) > R(\rho') \enspace.
\end{align}
\end{definition}

Definition~\ref{def:reward} does not impose any requirement on the reward if $r(\rho) = r(\hat{\rho})$, i.e., if two trajectories have the same rank. The choice of how the reward should serve as a tie-breaker in this event is left to the designer. In the next theorem we provide a candidate hierarchy-preserving reward and then rigorously show that it satisfies Definition~\ref{def:reward}.

\begin{theorem}[Rank-Preserving Reward Function] \label{thm:rank-preserve}
Let $a > 2$ and let $\rho_i\in[-a/2,a/2]$ for all $i\in\{1,2,\cdots,N\}$. Let the reward function $R$ be defined as follows:
\begin{equation}\label{eq:reward}
    R(\rho):= \sum_{i=1}^N \bigg(a^{N-i+1} \mathrm{step}(\rho_i) + \frac{1}{N} \rho_i\bigg) .
\end{equation}
Then $R$ satisfies Definition~\ref{def:reward}.
\end{theorem}
The proof of Theorem~\ref{thm:rank-preserve} is presented in the Appendix.
The key idea behind the construction of \eqref{eq:reward} is to ensure that the reward contribution on satisfaction of rule $i$ should exceed the sum of the reward contributions by all rules with lower priority. This is achieved by multiplying the step function with a constant that grows exponentially with the priority of a rule. To distinguish between trajectories that have the same rank, we use the average robustness for all $N$ rules as a criterion; note that the $\rho_i/N$ term in \eqref{eq:reward} sums up to the average robustness across all rules in the hierarchy.

\begin{remark}[Differentiable Reward] \label{rem:diff-reward}
    The reward \eqref{eq:reward} is not differentiable since it involves step functions. To facilitate continuous optimization using this reward, we approximate the step functions by sigmoids as follows:
    \begin{equation}\label{eq:reward-sigmoid}
        R(\rho):= \sum_{i=1}^N \bigg(a^{N-i+1} \mathrm{sigmoid}(c\rho_i) + \frac{1}{N} \rho_i\bigg) ,
    \end{equation}
    where $c>0$ is a scaling constants which is chosen to be large to mimic the step function.
\end{remark}

To provide more intuition on the reward function, in Fig.~\ref{fig:reward}, we plot the differentiable reward \eqref{eq:reward-sigmoid} of a 2-rule hierarchy (with $a=2.01$, $c=30$) by varying robustness $\rho_1$ and $\rho_2$ in $[-1,1]$. The quadrant where $\rho_1,\rho_2<0$ (neither rule satisfied) has the lowest reward while the quadrant where $\rho_1,\rho_2>0$ (both rules satisfied) has the highest reward. Of the remaining two quadrants, the one with $\rho_1>0$ (more important rule satisfied) has a higher reward than the one with $\rho_2>0$ (less important rule satisfied). These observations align well with Definition~\ref{def:reward}.

\begin{figure}[b]
  \centering
    \includegraphics[trim={2cm 1cm 2cm 2.5cm},clip,width=0.35\textwidth]{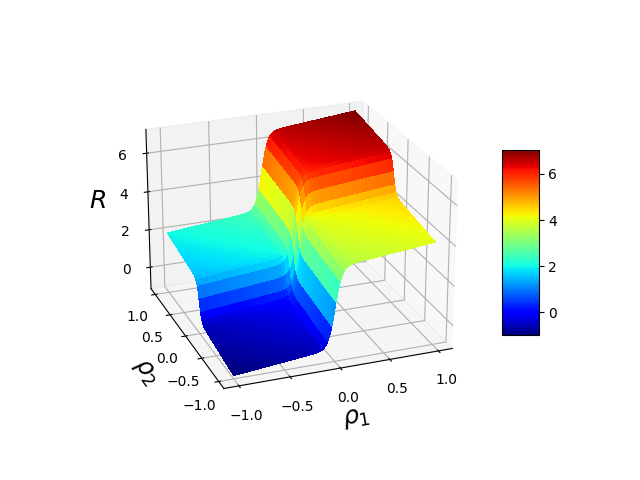}
  \caption{Visualization of the reward for a 2-rule hierarchy. \label{fig:reward}}
\end{figure}

\section{Receding Horizon Planning \\ With Rule Hierarchies}
\label{sec:motion-planner}

Equipped with a rule hierarchy and a method to cast it into a (nonlinear) differentiable reward function, we now present a two-stage algorithm to tractably solve (3). In the following sections, we describe the two stages: a coarse trajectory selection followed by a refinement process.

\subsection{Stage 1: Planning with Motion Primitives}
\label{subsec:planning-primitives-stage-1}
The objective of the first stage is to generate a coarse initial trajectory for warm-starting the continuous optimizer in the second stage; we note that such multi-stage planning is common in the literature \cite{rosolia2016autonomous,liniger2015optimization}. We achieve this by selecting the trajectory with the largest rule-hierarchy reward $R$ from a finite family of motion primitives. To generate the family of primitives $\mathcal{T}$, we use the approach presented in \cite{schmerling2018multimodal}. We first choose a set $\mathcal{M}$ of open-loop controls that coarsely cover a variety of actions that the AV can take; e.g., \texttt{(accelerate, turn right)}, \texttt{(decelerate, keep straight)}, etc. With the open-loop controls, the dynamics are propagated forward from the initial state for $\tau$ time-steps to generate $|\mathcal{M}|$ branches, where $|\mathcal{M}|$ is the cardinality of $\mathcal{M}$. At the terminal nodes of each of these $|\mathcal{M}|$ branches, all control actions are applied again for $\tau$ time-steps. This process is inductively repeated for the entire time horizon $T$ to produce a tree with $|\mathcal{M}|^{\lceil T/\tau \rceil}$ branches; see Fig.~\ref{fig:motion-primitive} for visualizing the motion primitives. The tree can be created efficiently by parallelizing the branch generation. 

\begin{figure}[b]
  \centering
    \includegraphics[trim={0 0 0 2mm},clip,width=0.40\textwidth]{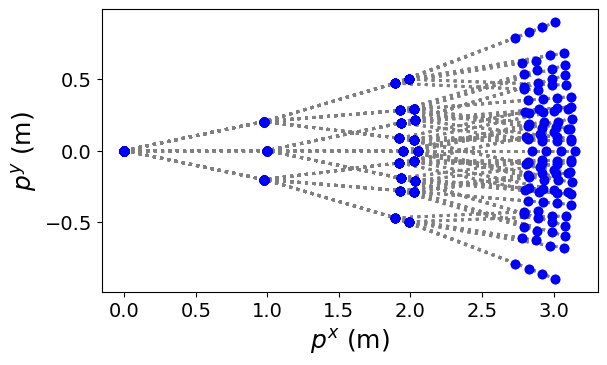}
  \caption{Motion primitive tree for $|\mathcal{M}|=6$, $T=3$, and $\tau=1$ at vehicle position $(0,0)$ m, heading $0$ rad, and speed $10$ m/s. \label{fig:motion-primitive}}
\end{figure}

\subsection{Stage 2: Continuous Trajectory Optimization}
\label{subsec:planning-continuous-stage-2}
The objective of the second stage is to refine the trajectory obtained from the first stage by solving the following optimization problem
\begin{align}\label{eq:cont-optimizer}
    \begin{aligned}
        \min_{u_{0:T}} & ~~~~ -R\circ\hat{\rho}(x_{0:T}, w) \\
        \text{s.t.} & ~~~~ x_{t+1} = f(x_t,u_t), ~\mathrm{for}~t=1,\cdots,T \enspace.
    \end{aligned}
\end{align}
We reiterate that the reward function is analytically differentiable due to Remark~\ref{rem:diff-reward} and the differentiability of the robustness of STL formulae afforded by STLCG \cite{leung2020back}. Hence, we compute the gradients analytically and use the Adam optimizer \cite{kingma2014adam} to solve \eqref{eq:cont-optimizer}. Owing to the lack of convexity of this optimization problem, we have no assurance of global optimality. Regardless, even convergence to a local optima in the vicinity of the initialization improves the trajectory's compliance to the rule hierarchy.

A psuedo-code for our two-stage motion planner is provided in Algorithm~\ref{alg:planner} below.
\begin{algorithm}[b]
\caption{Planning with a Rule Hierarchy \label{alg:planner}}
\small
\begin{algorithmic}[1]
	\State \textbf{Input:} Reward function $R$ for the rule hierarchy $\varphi$
    \State \textbf{Hyperparameters:} total planning time $\overline{T}$, planning horizon $T$, number of time-steps to execute $t_{\rm execute}$ 
    \State \textbf{Hyperparameters for Planning with Primitives:} set of open-loop controls for primitive tree generation $\mathcal{M}$, number of time-steps $\tau$ for which a control in $\mathcal{M}$ is executed 
    \State \textbf{Hyperparameters for Continuous Optimizer:} learning rate $\mathtt{lr}$, maximum iterations $K$, 
	\While{$t<\overline{T}$}
        \State $x_t \leftarrow$ \texttt{updateEgoState()}
        \State $w \leftarrow$ \texttt{updateWorldScene()}
        \State $\mathcal{T} \leftarrow$ \texttt{generatePrimitiveTree}$(x_t,\mathcal{M}, \tau, T)$
        \State $u_{t:t+T} \leftarrow \arg\max_{\mathcal{T}} R$
        \For{$k$ in $1:K$}
            \State Compute $-\nabla_{u_{t:t+T}}R$
    	    \State $u_{t:t+T} \leftarrow$ \texttt{Adam}$(u_{t:t+T}, -\nabla_{u_{t:t+T}}R, \mathtt{lr})$
         \EndFor
         \State \texttt{Execute}($u_{t:t+t_{\rm execute}}$)
        \State $t\leftarrow t+t_{\rm execute}$
    \EndWhile
	\end{algorithmic}
\normalsize
\end{algorithm}

\section{Experimental Evaluation}
\label{sec:results}

In this section we demonstrate the ability of our rule-hierarchy based receding horizon planner to navigate various complex driving scenarios while maintaining a high planning frequency. The simulations are performed in Python using the \texttt{highway-gym} \cite{highway-env} and tested on a desktop computer with an \texttt{AMD Threadripper Pro 3975WX} CPU and an \texttt{NVIDIA RTX 3090} GPU. We use PyTorch \cite{PyTorch} for motion planning to facilitate parallelization on GPU and analytical gradient computation via backpropagation and use STLCG \cite{leung2020back} to encode STL formulae. We have made public our code for building general-purpose STL rule hierarchies with the functionality to compute the rank-preserving reward at: \url{https://github.com/NVlabs/rule-hierarchies}.

\noindent \textbf{Dynamics.} We use the kinematic bicycle model \cite{kong2015kinematic} for the ego vehicle's dynamics. The state $x:=(p^x, p^y, \psi, v)$ comprises of the position $(p^x, p^y)$ of the vehicle's center in an inertial frame of reference, its orientation $\psi$ measured with respect to the horizontal axis of the inertial frame, and its heading velocity in the vehicle's frame of reference. The control inputs $u:=(\alpha,\delta)$ comprise of the acceleration $\alpha$ and the steering angle $\delta$; see \cite[Section~II]{kong2015kinematic} for further details.

\noindent \textbf{Hierarchy-preserving Reward Function.} We use the reward function \eqref{eq:reward-sigmoid} described in Remark~\ref{rem:diff-reward} with $a=2.01$. The robustness values directly returned by STLCG \cite{leung2020back} do not necessarily lie within $[-a/2,a/2]$, as is required by Theorem~\ref{thm:rank-preserve}. To remedy this, we scale the robustness values using $\tanh$. Let $\rho_{\rm STLCG}$ be the robustness of an STL formula provided by STLCG. Then we use $\rho=\tanh(\rho_{\rm STLCG}/s)$ as the robustness in our reward function \eqref{eq:reward-sigmoid}, where $s>0$ is a scaling constant chosen based on the range of values that $\rho_{\rm STLCG}$ takes. We remark that in our study, choosing $s$ was straightforward and did not require much tuning.

\noindent \textbf{Planning.} For the first planning stage, the motion primitives are generated using the method described in Section~\ref{subsec:planning-primitives-stage-1}. We choose $\mathcal{M}$ as the mesh grid generated by the Cartesian product of the accelerations $\{-5, 5\}~\mathrm{m/s}^2$ with the steering angles $\{-\pi/8, 0, \pi/8\}~\mathrm{rad}$; consequently, $|\mathcal{M}|=6$. The motion primitive tree is generated for $T=10$ and $\tau=2$, and therefore, has $6^5$ branches. The tree generation and choosing the branch with the highest reward takes less than $0.04~\mathrm{s}$ due to parallelization on the GPU. A sample motion primitive tree is plotted in Fig.~\ref{fig:motion-primitive} for visualization.
As described in Section~\ref{subsec:planning-continuous-stage-2}, the second-stage of the planner performs continuous trajectory optimization using Adam \cite{kingma2014adam} to solve \eqref{eq:cont-optimizer}. We set the learning rate at $0.01$ and the maximum iterations $K$ (see line 10 in Algorithm~\ref{alg:planner}) as $10$. For all forthcoming examples, we use a planning horizon of $T=10$ and execute only the first action (i.e. $t_{\rm execute}=1$; see line 15 in Algorithm~\ref{alg:planner}) from the plan before replanning.

\subsection{Road Navigation}
\label{subsec:road-navigation}

We use six rules in the rule hierarchy, listed in Table~\ref{tab:highway-rule-hierarchy}, for the scenarios in Fig.~\ref{fig:road-nav}. The highest priority rule is collision avoidance which requires the ego's position to lie outside a $10~\mathrm{m}\times 4~\mathrm{m}$ rectangluar zone around both non-ego vehicles. The second and third rules incentivize the ego vehicle to not cross the solid-white lane line on its right and the dashed-white lane line on its left, respectively. Not crossing the solid-white lane line is prioritized over dashed-white lane line in accordance with common traffic laws. The fourth rule incentivizes the ego to be aligned with the lane by the end of the planning horizon to prevent oscillations about the lane center. Finally, the last two rules require the ego to maintain a lower and upper speed limit. The lower speed limit ensures that the ego vehicle only stops if necessary.

\subsubsection{Overtake from lane} In this scenario, shown in Fig.~\ref{subfig:lane-overtake}, the yellow ego vehicle is moving at a sufficiently high speed that by the time it observes the stationary orange vehicle, it cannot stop in time to avoid a collision. Therefore, a lane change is the only option to avoid a collision with the orange car. The blue non-ego vehicle is significantly faster than the ego vehicle, resulting in a gap on the left lane. Hence, our planner relaxes rule number 3 in Table~\ref{tab:highway-rule-hierarchy}.

\subsubsection{Overtake from shoulder} This scenario, shown in Fig.~\ref{subfig:shoulder-overtake}, is similar to the scenario discussed above, with the salient difference being that the blue non-ego vehicle moves at a similar speed as the yellow ego vehicle. Hence, changing the lane is not viable as it would result in a collision with the blue vehicle. Therefore, our planner relaxes rule 2 in Table~\ref{tab:highway-rule-hierarchy} to overtake the orange car from the shoulder.

\subsubsection{Stop instead of overtake} In this scenario, shown in Fig.~\ref{subfig:stop-instead-of-overtake}, the only difference from the scenario in Fig.~\ref{subfig:shoulder-overtake} is that the ego vehicle is moving slow enough to be able to stop before colliding with the orange vehicle. In compliance with the rule hierarchy in Table~\ref{tab:highway-rule-hierarchy}, violating the lower priority speed rules is preferable to changing lanes, hence, the planner relaxes rule 5 and brings the ego vehicle to a halt.

\subsubsection{Double-parked vehicle} In this scenario, shown in Fig.~\ref{subfig:double-parked}, the stationary orange vehicle is not blocking the entire lane, but is, instead, double-parked. The ego vehicle is able to navigate around the orange vehicle without ever leaving its lane. Our planner is able to find a trajectory that complies with all the rules in the rule hierarchy.

\begin{table}
    \centering
    \begin{tabular}{cc}
        \toprule
        \textbf{Priority} & \textbf{Rule Description} \\ 
        \midrule
        1 & No Collision \\ 
        2 & Do not cross solid lane line \\ 
        3 & Do not cross dashed lane line \\ 
        4 & Orient along the lane by the end of the planning horizon \\ 
        5 & Speed $\geq 2~\mathrm{m/s}$ \\ 
        6 & Speed $\leq 15~\mathrm{m/s}$ \\ 
        \bottomrule
    \end{tabular}
    \caption{Rule hierarchy for road navigation scenarios. \label{tab:highway-rule-hierarchy}}
    \vspace{-3mm}
\end{table}

\begin{figure*}[t]
  \centering
  \subfigure[Overtake from adjacent lane]
  {
    \includegraphics[trim={0 0 0 1mm},clip,width=0.47\textwidth]{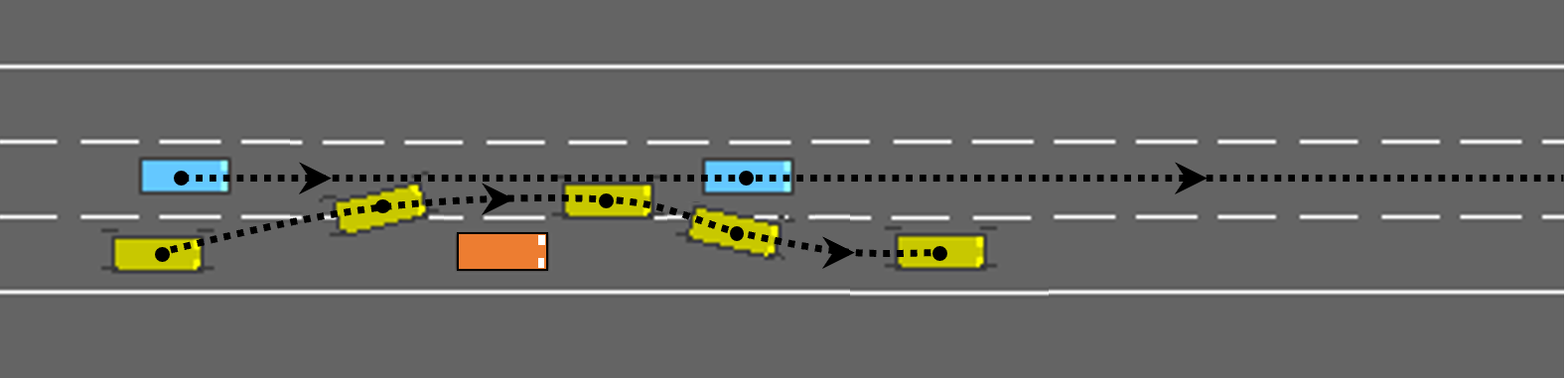}
    \label{subfig:lane-overtake}
  }
  \subfigure[Overtake from shoulder]
  {
    \includegraphics[width=0.47\textwidth]{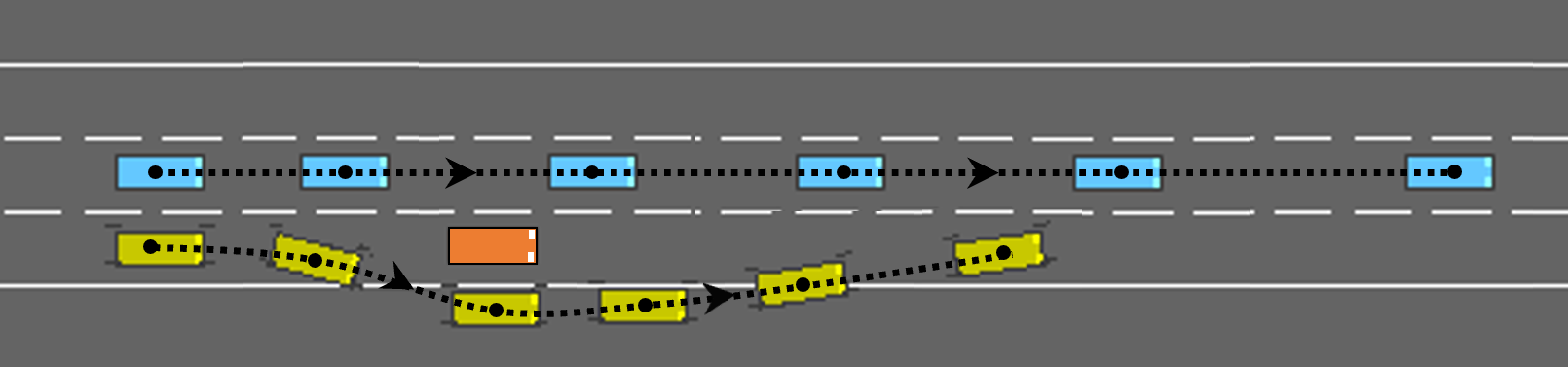}
    \label{subfig:shoulder-overtake}
  }
    \subfigure[Stop instead of overtake]
  {
    \includegraphics[width=0.47\textwidth]{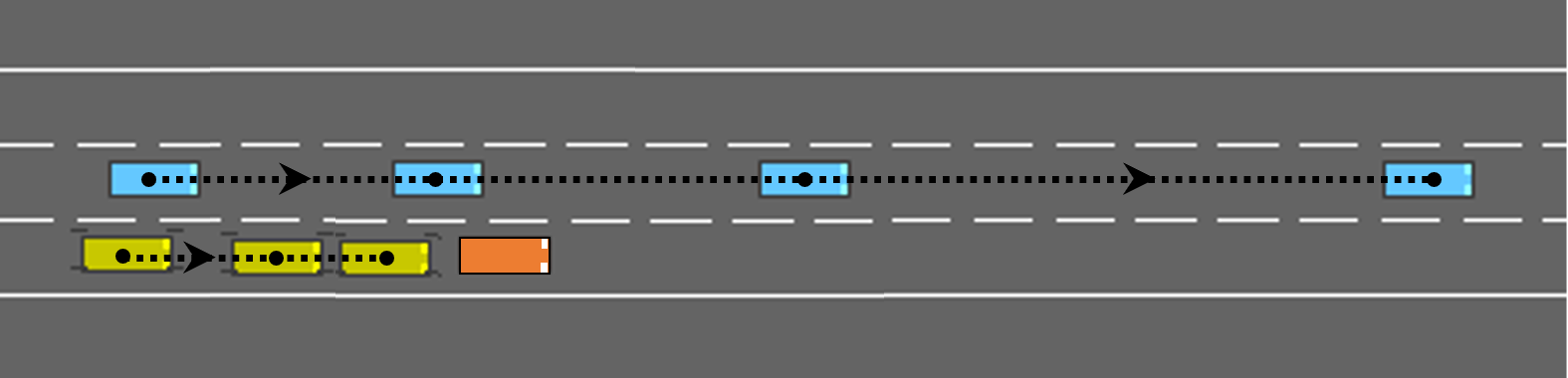}
    \label{subfig:stop-instead-of-overtake}
  }
    \subfigure[Double-parked vehicle]
  {
    \includegraphics[width=0.47\textwidth]{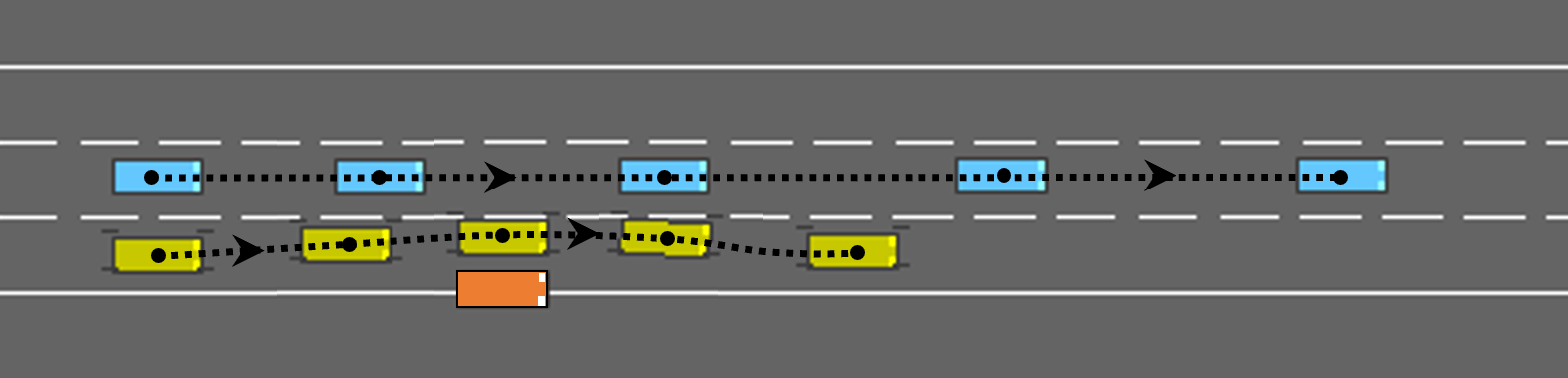}
    \label{subfig:double-parked}
  }
  \vspace{-2mm}
  \caption{Snapshots of road navigation scenarios. The ego vehicle is in yellow, the moving non-ego vehicle is in blue, and the stationary non-ego vehicle is in orange. Dotted lines indicate trajectories of moving agents in the direction indicated by the arrows. \label{fig:road-nav}}
\end{figure*}

\subsection{Intersection Negotiation}
\label{subsec:intersection-navigation}

The rule hierarchy in this example includes seven rules, six of which are identical to the rule hierarchy used in Section~\ref{subsec:road-navigation}. The only new addition is a stop-sign rule that requires the ego vehicle to stop for 1 second in front of a stop sign (denoted by the transparent red square in Fig.~\ref{fig:intersection-navigation}) before moving on. The rule hierarchy is listed in Table~\ref{tab:intersection-rule-hierarchy}. It is worth noting that the stop-sign rule conflicts with the minimum-speed rule; nonetheless, when the ego vehicle approaches a stop sign, the planner will violate the minimum-speed rule in favor of the stop-sign rule owing to their priorities in the rule hierarchy.

\subsubsection{Wait} In this scenario, shown in Fig.~\ref{subfig:intersection-wait}, the ego vehicle in yellow must stop at a stop sign for 1 second before crossing an intersection. A blue non-ego vehicle is driving on the perpendicular lane which has no stop sign. By the end of the ego vehicle's 1-second stop, the blue vehicle is already in the intersection, so the ego vehicle waits for the blue vehicle before crossing the intersection.

\subsubsection{Go} In this scenario, shown in Fig.~\ref{subfig:intersection-go}, the blue vehicle is still very far from the intersection by the time the ego vehicle finishes its 1-second stop at the stop sign, unlike the previous scenario. Therefore, the ego vehicle proceeds to go through the intersection first.

\begin{figure*}[t]
  \centering
  \subfigure[Wait]
  {
    \includegraphics[trim={0 1cm 0 5mm},clip,width=0.47\textwidth]{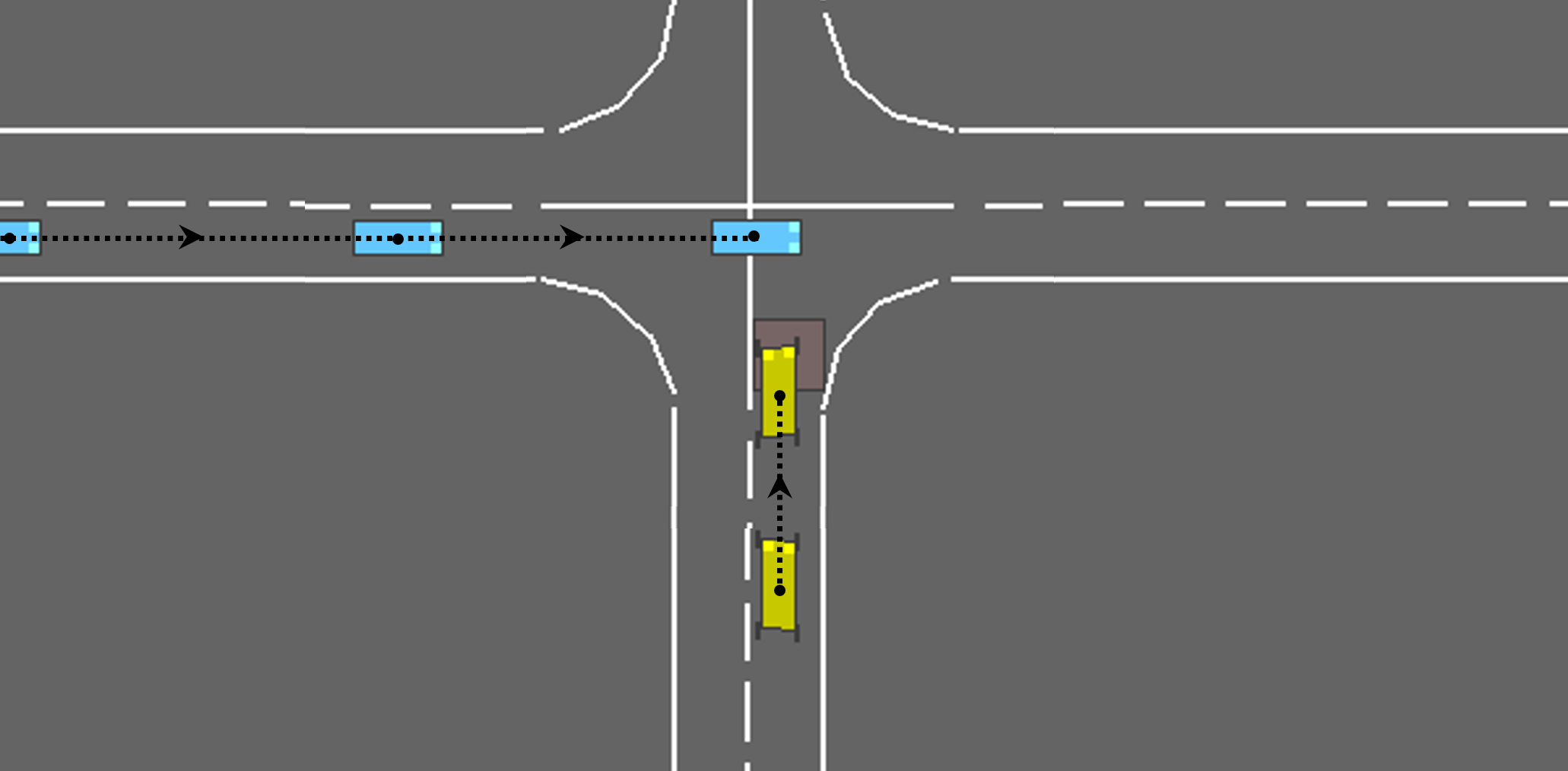}
    \label{subfig:intersection-wait}
  }
  \subfigure[Go]
  {
    \includegraphics[trim={0 1cm 0 5mm},clip,width=0.47\textwidth]{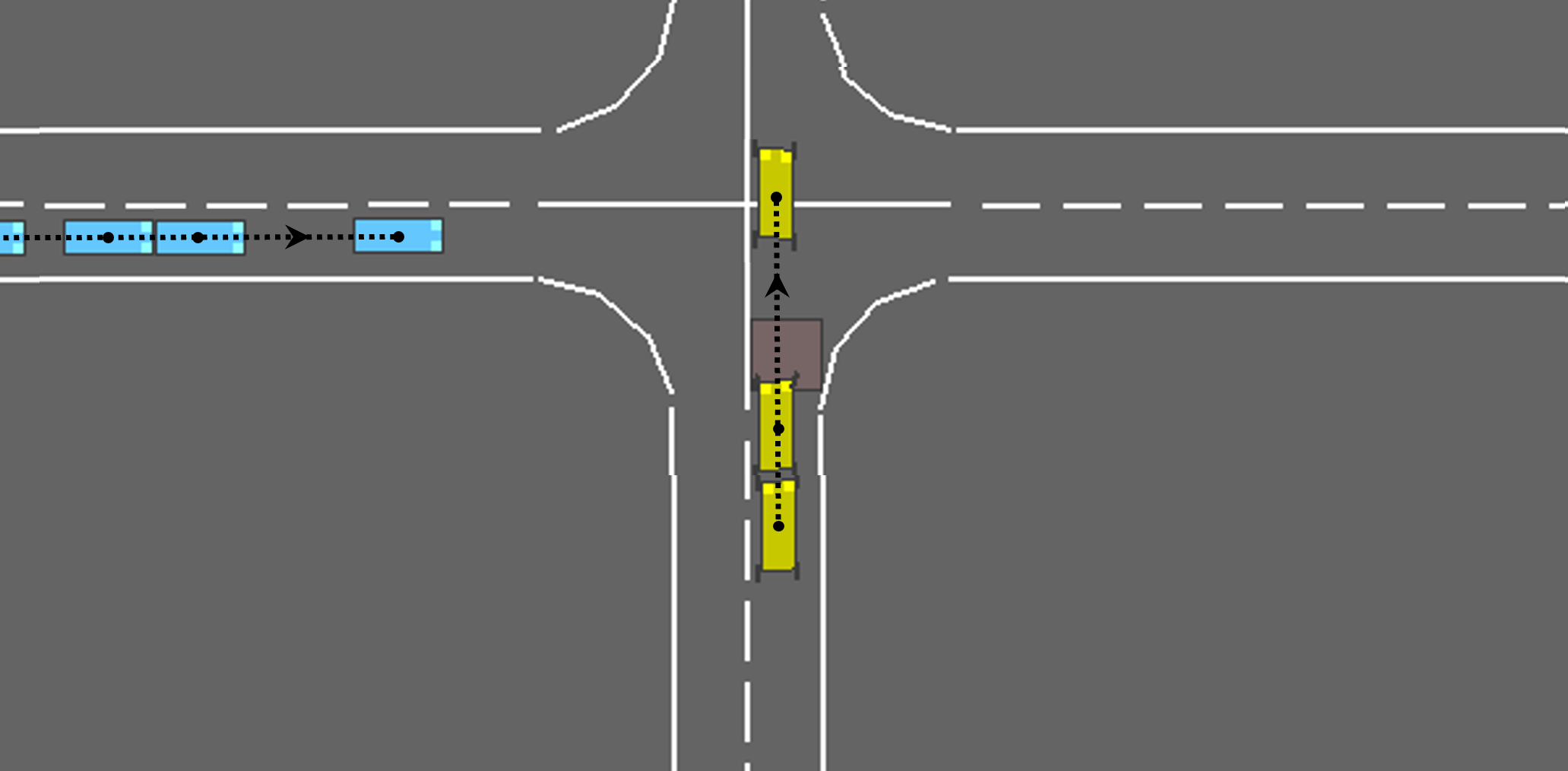}
    \label{subfig:intersection-go}
  }
  \vspace{-2mm}
  \caption{Snapshots of intersection navigation scenarios. The ego vehicle is in yellow and the non-ego vehicle is in blue. The transparent red-box is the stopping zone of the stop-sign. Dotted lines indicate trajectories of moving agents in the direction indicated by the arrows. \label{fig:intersection-navigation}}
  \vspace{-3mm}
\end{figure*}

\subsection{Discussion}
We observe from Table~\ref{tab:time-stats} that our planner, even in the worst-case, is able to generate motion plans within $0.1$ s for road navigation scenarios and within $0.14$ s for intersection negotiation scenarios.
The key reason behind the real-time performance of our planner is the use of the rank-preserving reward function \eqref{eq:reward-sigmoid}. Instead of going through $2^N$ rule-satisfaction combinations (in the worst case), the reward function \eqref{eq:reward-sigmoid} allows us to plan with just two optimizations per planning cycle, as described in Section~\ref{sec:motion-planner}. Consequently, our planning times also exhibit low standard deviation; note that all four road navigation scenarios and both the intersection negotiation scenarios have comparable planning times.

Another noteworthy point about our approach is the robustness of hyperparameters. We emphasize that all hyperparameters---robustness scaling constant, sigmoid sharpness, learning rate, etc.---were same across all scenarios in both settings, road and intersection navigation; besides the inclusion of a new stop-sign rule in the rule-hierarchy for intersection navigation. Finally, we emphasize again that planning with rule hierarchies facilitates complex decision making and seamless adaptation to a rich class of scenarios.

\begin{table}
    \centering
    \begin{tabular}{cc}
        \toprule
        \textbf{Priority} & \textbf{Rule Description} \\ 
        \midrule
        1 & No Collision \\ 
        2 & Do not cross solid lane line \\ 
        3 & Do not cross dashed lane line \\ 
        4 & Stop at stop sign for at least 1 second \\
        5 & Orient along the lane by the end of the planning horizon \\ 
        6 & Speed $\geq 2~\mathrm{m/s}$ \\ 
        7 & Speed $\leq 15~\mathrm{m/s}$ \\ 
        \bottomrule
    \end{tabular}
    \caption{Rule hierarchy for intersection negotiation scenarios. \label{tab:intersection-rule-hierarchy}}
\end{table}

\begin{table}
    \centering
    \resizebox{0.5\textwidth}{!}{
        \begin{tabular}{ccccc}
            \toprule
            \textbf{Scenario} & \textbf{Mean$\pm$Std} (s) & \textbf{Median} (s) & \textbf{Max} (s) & \textbf{Min} (s) \\ 
            \midrule
            Overtake from lane & $0.094\pm 0.002$ & 0.095 & 0.097 & 0.086 \\ 
            Overtake from shoulder & $0.092\pm 0.001$ & 0.092 & 0.094 & 0.086 \\ 
            Stop instead of overtake & $0.091\pm 0.003$ & 0.092 & 0.095 & 0.078 \\ 
            Double-parked vehicle & $0.093\pm 0.001$ & 0.093 & 0.095 & 0.085 \\ 
            Intersection: wait & $0.105\pm 0.017$ & 0.091 & 0.127 & 0.085 \\  
            Intersection: go & $0.104\pm 0.019$ & 0.090 & 0.133 & 0.082 \\  
            \bottomrule
        \end{tabular}
    }
    \caption{Time statistics for planning. The statistics are computed for all planning cycles within a single run of a particular scenario.  \label{tab:time-stats}}
    \vspace{-4mm}
\end{table}

\section{Conclusions and Future Work}
\label{sec:conclusions}

In this paper, we presented a framework for online motion planning of AVs with rule hierarchies. We achieved this by formulating a hierarchy-preserving reward function that allowed us to plan without checking all combinations of rule satisfaction. The rule hierarchies were expressed as STL formulae which allowed us to incorporate temporal constraints, such as wait at the stop-sign before driving on. Finally, we demonstrated the ability of our planner to seamlessly adapt in road navigation and intersection negotiation scenarios.

This work opens up various exciting future directions. As an immediate future direction, we intend to perform closed-loop evaluation of our planner with data-driven trajectory predictors \cite{kamenev2022predictionnet,salzmann2020trajectron++} in state-of-the-art traffic simulators \cite{xu2022bits,nuplan}. Another interesting future direction is to imbue our planner with the ability to reason about estimation and prediction uncertainties by translating them to uncertainties in the rank induced by the rule-hierarchy. Finally, we note that rule hierarchies can model human driving behavior well, as suggested by \cite{helou2021reasonable}; hence, we hope to leverage rule hierarchies for estimating the degree of responsibility and responsiveness exhibited by other non-ego vehicles.

\section*{Appendix}
\label{sec:app}

\begin{proof}[Proof of Theorem~\ref{thm:rank-preserve}]
Let $\rho$ and $\rho'$ be the given robustness vectors for which $r(\rho) < r(\rho')$. Let the first index at which $\rho$ satisfies a rule that is not satisfied by $\rho'$ be
\begin{equation}\label{eq:highest-no-match-idx}
    k := \min\{i~|~\text{step}(\rho_i) > \text{step}(\rho'_i),~i\in \{1,2,\cdots,N\} \}.
\end{equation}
Now, decompose the reward function $R$ as follows:
\begin{align}
    R(\rho) & = \underbrace{\sum_{i=1}^{k-1} a^{N-i+1} \text{step}(\rho_i)}_{=:b} + a^{N-k+1} \nonumber\\
            & + \sum_{i=k+1}^N a^{N-i+1} \text{step}(\rho_i) + \frac{1}{N}\sum_{j=1}^N \rho_j . \nonumber
\end{align}
The first term in the above, defined as a constant $b$, is the same for both $\rho$ and $\rho'$. Since, $\text{step}(\rho_k)$ can only be $0$ or $1$, it follows from \eqref{eq:highest-no-match-idx} that $\text{step}(\rho_k) = 1$ while $\text{step}(\rho'_k)=0$. Hence, the reward for $\rho'$ can be decomposed as follows:
\begin{align}\label{eq:R-rho'}
    R(\rho') = b + 0 + \sum_{i=k+1}^N a^{N-i+1} \text{step}(\rho'_i) + \frac{1}{N}\sum_{j=1}^N \rho'_j .
\end{align}

We make the following claim:\\
\textbf{Claim 1:} $\sum_{i=k+1}^N a^{N-i+1} \text{step}(\rho'_i) < a^{N-k+1} - a$.

Now we will prove this claim. Consider
\begin{align}\nonumber
    \sum_{i=k+1}^N a^{N-i+1} \text{step}(\rho'_i) \leq \sum_{i=k+1}^N a^{N-i+1} = \frac{a(a^{N-k}-1)}{a-1} .
\end{align}
Using $a > 2 \iff a-1 > 1 \iff 1/(a-1) < 1$ above:

\vspace{-3mm}
\small
\begin{align}
    & \sum_{i=k+1}^N a^{N-i+1} \text{step}(\rho'_i) \leq \frac{a(a^{N-k}-1)}{a-1} < a(a^{N-k} - 1) ,
\end{align}
\normalsize
completing the proof of Claim 1.

Now we make a second claim:\\
\textbf{Claim 2:} $\frac{1}{N} \sum_{j=1}^N\rho'_j - a \leq \frac{1}{N} \sum_{j=1}^N\rho_j$.

The proof for this claim immediately follows by using $\rho_j,\rho'_j\in [-a/2,a/2]$ as follows:
\begin{align}
    \frac{1}{N} \sum_{j=1}^N\rho'_j -  \frac{1}{N} \sum_{j=1}^N\rho_j \leq \frac{1}{N} \bigg(\frac{aN}{2} + \frac{aN}{2}\bigg) = a
\end{align}

With both these claims established, use Claim 1 in \eqref{eq:R-rho'} followed by Claim 2 to get
\begin{align}
    R(\rho') < b + a^{N-k+1} + \frac{1}{N}\sum_{j=1}^N \rho'_j - a \leq R(\rho) , \nonumber
\end{align}
completing the proof of this theorem.
\end{proof}

\section*{Acknowledgment}
We are grateful to Ryan Holben for helpful discussions on planning with rule hierarchies and to Wei Xiao for discussing his prior work on optimal control with rule hierarchies.

\addtolength{\textheight}{-4.9cm}   %

\bibliographystyle{ieeetran}
\bibliography{bib}

\end{document}